\documentclass{article}
\IfFormatAtLeastTF{2024-06-01}{}{\RequirePackage{array}[=2020-02-10]}
\usepackage[table]{xcolor}
\usepackage{iclr2027_conference,times}

\usepackage{amsmath,amsfonts,bm}

\def\eqref#1{equation~\ref{#1}}

\def\plaineqref#1{\ref{#1}}

\def\1{\bm{1}}

\DeclareMathAlphabet{\mathsfit}{\encodingdefault}{\sfdefault}{m}{sl}
\SetMathAlphabet{\mathsfit}{bold}{\encodingdefault}{\sfdefault}{bx}{n}

\newcommand{\E}{\mathbb{E}}

\newcommand{\Var}{\mathrm{Var}}

\usepackage[hypertexnames=false]{hyperref}
\hypersetup{
    colorlinks=true,
    linkcolor=[rgb]{0.10,0.20,0.55},
    citecolor=[rgb]{0.10,0.20,0.55},
    urlcolor=[rgb]{0.10,0.20,0.55}
}
\usepackage{url}
\usepackage{amsmath,amssymb,amsthm}
\usepackage{booktabs}
\definecolor{gacablue}{RGB}{31,84,150}
\definecolor{gacaorange}{RGB}{214,110,32}
\definecolor{gacagray}{RGB}{150,158,170}
\definecolor{gacashade}{gray}{0.925}
\usepackage{algorithm}
\usepackage{algpseudocode}
\usepackage{multirow}
\usepackage{array}
\usepackage{graphicx}
\usepackage{placeins}
\graphicspath{{figures/}}

\newtheorem{theorem}{Theorem}
\newtheorem{corollary}[theorem]{Corollary}

\newtheorem{proposition}[theorem]{Proposition}
\theoremstyle{definition}

\theoremstyle{remark}

\DeclareMathOperator{\clip}{clip}
\providecommand{\E}{\mathbb{E}}
\providecommand{\Var}{\operatorname{Var}}

\newcommand{\Aep}{A^{\mathrm{E}}}                 
\newcommand{\Ast}{A^{\mathrm{S}}}                 
\newcommand{\Afin}{A^{\mathrm{GACA}}}             
\newcommand{\Hstep}{H}                            
\newcommand{\crit}{c}                             
\newcommand{\lammix}{\lambda}                     
\newcommand{\sd}[1]{{\scriptsize$\pm#1$}}         

\title{
Granularity-Adaptive Credit Assignment for Long-Horizon LLM Agent Reinforcement Learning
}

\author{%
\textbf{Taoran Liang$^{1,2,*}$ \quad
Yang Liu$^{3,*}$ \quad
Shang Luo$^{2}$ \quad
Yingguang Yang$^{2}$} \\
\textbf{Rongrong Zhang$^{4}$ \quad
Yingzong Min$^{4}$ \quad
Yulin Huang$^{3}$ \quad
Jianshen Zhang$^{3}$} \\
\textbf{Yongzhi Qi$^{3}$ \quad
Kefu Xu$^{2}$ \quad
Congjing Ran$^{5}$ \quad
Bin Pan$^{1}$ \quad
Bin Chong$^{2,\dagger}$} \\[0.8ex]
{\normalfont\small
$^{1}$Nankai University \quad
$^{2}$Peking University \quad
$^{3}$Supply Chain Tech Team Y, JD.com
} \\
{\normalfont\small
$^{4}$Shanghai Waybot Technology Co., Ltd. \quad
$^{5}$Wuhan University
} \\[0.4ex]
{\normalfont\small
$^{*}$Equal contribution
\qquad
$^{\dagger}$Corresponding author
}
}

\iclrfinalcopy
\hypersetup{
    pdftitle={Granularity-Adaptive Credit Assignment for Long-Horizon LLM Agent Reinforcement Learning},
    pdfauthor={Taoran Liang, Yang Liu, Shang Luo, Yingguang Yang, Rongrong Zhang, Yingzong Min, Yulin Huang, Jianshen Zhang, Yongzhi Qi, Kefu Xu, Congjing Ran, Bin Pan, Bin Chong}
}

\begin{document}

\maketitle
\lhead{Preprint}

\begin{abstract}
Long-horizon language-model agents trained with reinforcement learning often receive sparse outcome rewards that do not reveal which decisions along a trajectory deserve credit. Episode-level advantages provide coarse trajectory-wide credit, while step-level comparisons offer finer resolution with context-dependent estimation noise. We propose Granularity-Adaptive Credit Assignment (GACA), a critic-free method that adaptively mixes episode- and step-level credit for each decision during policy optimization. GACA normalizes the sampled response’s mean per-token negative log-likelihood (NLL) within each trajectory and uses the resulting criticality score to determine the step-specific mixture. The computation reuses rollout log-probabilities without additional training rollouts or model evaluations. Our analysis characterizes optimal score-dependent mixing and derives conditions linking expected NLL to a lower bound on the preferred step-level weight. Across ALFWorld and WebShop with 1.5B and 7B backbones, GACA achieves the highest reported mean success rates among the compared methods, while introducing negligible additional computation.
\end{abstract}

\section{Introduction}
\label{sec:intro}
Language-model agents act over many steps to complete household tasks, browse web stores, and use external tools~\citep{yao2022react,shridhar2020alfworld,yao2022webshop}. Reinforcement learning can improve these agents using task outcomes~\citep{ouyang2022training,guo2025deepseek}, but a sparse terminal reward leaves a difficult credit-assignment problem: which decisions should receive the learning signal? A trajectory can contain both consequential choices and actions whose alternatives would have little effect on the outcome~\citep{arjona2019rudder}.

Critic-free methods avoid training a value model by comparing several rollouts of the same task. GRPO~\citep{shao2024deepseekmath} computes an episode-level relative advantage and broadcasts it to every action. GiGPO~\citep{feng2026group} adds local resolution by grouping steps with the same anchor observation and comparing their continuation returns. Its fixed combination of episode- and step-level signals leaves a further question: \emph{should every decision rely on these two credit estimates in the same proportion?}

The two granularities offer a decision-dependent estimation trade-off. Episode credit spreads the outcome over all decisions, while step credit separates local continuations using finite, context-dependent comparison groups. We propose Granularity-Adaptive Credit Assignment (GACA), which adapts their relative contribution during policy optimization.

For each completed rollout, GACA computes the sampled response's mean per-token negative log-likelihood (NLL) from old-policy log-probabilities. Dividing by the trajectory mean gives a criticality score that measures how surprising the response is relative to other responses in that interaction. A bounded mapping converts this score into a step-specific mixing weight, placing more weight on step-level credit at higher scores. These computations reuse quantities already available for the policy update, requiring no additional model or training rollouts.

Our analysis characterizes how a score-dependent mixture balances the errors of the two credit estimates. It then identifies conditions under which greater variation in action values favors step-level credit and expected NLL provides a lower bound on that variation. These results motivate increasing the weight of local credit with the NLL-based score.

GACA complements methods that refine the step estimator, including history-consistent grouping in HGPO~\citep{he2026hierarchy}, behavior-aware grouping in BiPACE~\citep{wang2026bipace}, and progress-conditioned credit in ProGPO~\citep{yang2026progress}. Uncertainty-based modulation in AEM and STARE rescales an existing advantage~\citep{zhao2026aem,luo2026stare}; GACA uses uncertainty to control the relative contribution of two credit granularities. Appendix~\ref{app:related} discusses these connections.

Our evaluation covers ALFWorld and WebShop with Qwen2.5-Instruct at 1.5B and 7B scales. On ALFWorld 1.5B, matched ablations compare GACA with five fixed mixtures and within-trajectory random score assignment. Adaptive mixing reaches $96.88\%$ success, compared with $93.05\%$ for the best tested fixed mixture and $88.54\%$ for random assignment. These comparisons examine both the choice of mixing weights and their assignment to individual decisions.

\section{Preliminaries}
\label{sec:prelim}
\subsection{Sequential interaction and policy optimization}
\label{sec:mdp}
At step $t$, the full pre-action state $s_t$ contains the task, interaction history, and current observation. The agent samples a response $a_t\sim\pi_\theta(\cdot\mid s_t)$ containing reasoning and an action command. The environment executes its parsed action $U_t=\phi(a_t)$, returns reward $r_t$, and continues until termination at $T$. The trajectory return is $R(\tau)=\sum_t r_t$, and the discounted step return is $\hat g_t=\sum_{k=t}^T\gamma^{k-t}r_k$. Each task produces a group of $G$ trajectories for relative advantage estimation.

A per-step advantage $\Afin_t$ is broadcast to the response's valid tokens. With token importance ratio $\rho_{t,j}=\pi_\theta(a_t^j\mid s_t,a_t^{<j})/\pi_{\theta_{\rm old}}(a_t^j\mid s_t,a_t^{<j})$, training uses the clipped surrogate~\citep{schulman2017proximal}
\begin{equation}
\mathcal L_{\rm PG}=-\E_t\!\left[\frac1{|a_t|}\sum_j
\min\!\left(\rho_{t,j}\Afin_t,\clip(\rho_{t,j},1-\epsilon,1+\epsilon)\Afin_t\right)\right]
+\beta_{\rm KL}\widehat{\mathrm{KL}}(\pi_\theta\|\pi_{\rm ref}).
\label{eq:surrogate}
\end{equation}
GACA supplies the advantage used in this update. For ALFWorld 1.5B, the expectation uses response weights $w_t=|a_t|/\sum_{k\in\mathcal B}|a_k|$ in optimization batch $\mathcal B$, giving a mean over valid tokens. Appendix~\ref{app:impl} specifies loss reduction and regularization.

\subsection{Episode- and step-level credit}
\label{sec:twolevel}
The episode-side advantage compares the outcome score $\widetilde R_t=R(\tau)-p_t$ with its prompt group $g$, where $p_t$ is the configured local invalid-action penalty:
\begin{equation}
\Aep_t=\frac{\widetilde R_t-\mu_g}{\sigma_g+\varepsilon}.
\label{eq:episode}
\end{equation}
The trajectory outcome is shared across decisions, with the validity adjustment applied before normalization. Group statistics are computed over step rows; Appendix~\ref{app:impl} gives the conventions.

Following GiGPO, steps with identical anchor observations form a group $\mathcal G$. Their discounted returns give the step advantage
\begin{equation}
\Ast_t=\frac{\hat g_t-\mu_{\mathcal G}}{\sigma_{\mathcal G}+\varepsilon},
\qquad \mu_{\mathcal G}=|\mathcal G|^{-1}\sum_{k\in\mathcal G}\hat g_k.
\label{eq:step}
\end{equation}
The standard-deviation division in both equations is used for ALFWorld and omitted for WebShop. A singleton group uses prompt-group normalization of its step return. Anchor matching uses observations and can group steps with different histories; the resulting estimation error is included in the analysis below. GiGPO combines the two signals as $\Aep+\omega\Ast$ with a fixed coefficient $\omega$.

\section{Granularity-Adaptive Credit Assignment}
\label{sec:method}
During advantage computation, GACA assigns a mixing weight to each environment step using a criticality score computed from the completed rollout (Figure~\ref{fig:overview}). It reuses the two group-relative advantages and the logged token probabilities, adding only scalar score and mixing operations.

\begin{figure}[t]
\centering
\includegraphics[width=\textwidth]{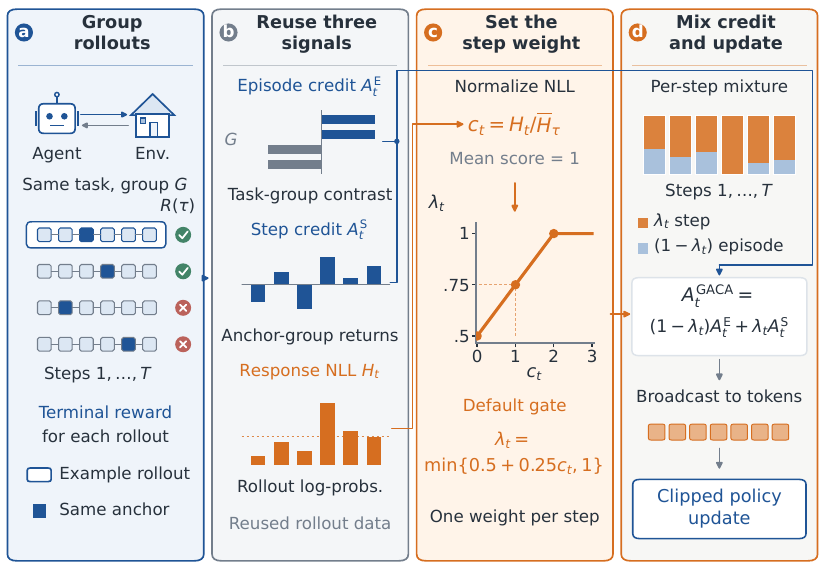}
\caption{GACA during policy optimization. Completed rollout groups provide episode credit, anchor-grouped step credit, and sampled-response NLL. Trajectory normalization gives $c_t$, and the default gate assigns the step-level weight $\lambda_t=\min\{0.5+0.25c_t,1\}$. The mixed advantage is broadcast to response tokens for the clipped policy update. Trajectories and signal values are schematic.}
\label{fig:overview}
\end{figure}

\subsection{An NLL-based criticality score}
\label{sec:criticality}
For the full response sampled at step $t$, including reasoning and the action command, define its mean per-token NLL and trajectory-normalized criticality as
\begin{align}
H_t&=-\frac1{|a_t|}\sum_{j=1}^{|a_t|}\log\pi_{\theta_{\rm old}}(a_t^j\mid s_t,a_t^{<j}),\label{eq:nll}\\
c_t&=\frac{H_t}{\overline H_\tau},\qquad
\overline H_\tau=\frac1T\sum_{k=1}^T H_k.\label{eq:crit}
\end{align}
The criticality score is an NLL-based proxy for allocating credit: $c_t>1$ identifies a response that is more surprising than the trajectory average. For $\overline H_\tau>0$, the scores have mean one and preserve the ordering of $H_t$ within the trajectory. Normalization puts decisions from the same interaction on a common relative scale. Scores are computed after rollout completion, when the full trajectory is available for advantage estimation.

\subsection{Bounded adaptive mixing}
\label{sec:mixing}
We map criticality to a bounded step-level weight and combine the two advantages:
\begin{align}
\lambda_t&=\min\{\lambda_0(1+\alpha_\lambda c_t),1\},
\qquad \lambda_0\in[0,1],\quad\alpha_\lambda\ge0,\label{eq:lambda}\\
\Afin_t&=(1-\lambda_t)\Aep_t+\lambda_t\Ast_t.\label{eq:final}
\end{align}
The parameter $\lambda_0$ sets the step-weight floor, and $\alpha_\lambda$ controls the increase with criticality before saturation. Setting $\alpha_\lambda=0$ gives a fixed mixture. The defaults $\lambda_0=\alpha_\lambda=0.5$ yield $\lambda_t=0.75$ at $c_t=1$ and reach the upper cap at $c_t=2$. The gate retains a positive step-level coefficient at low criticality and assigns all weight to the step advantage once $c_t\ge2$.

Appendix~\ref{app:suppablation} reports sensitivity to the modulation strength. Rollout-derived scores and advantages are held fixed during each policy update. Algorithm~\ref{alg:gaca} summarizes the computation; Appendix~\ref{app:impl} gives the normalization and singleton conventions.

\begin{algorithm}[t]
\caption{GACA advantage computation}
\label{alg:gaca}
\begin{algorithmic}[1]
\Require group rollouts $\{\tau_i\}_{i=1}^{G}$ with returns, step returns, action-validity flags, anchor observations, and token log-probabilities; hyperparameters $\lambda_0,\alpha_\lambda$
\State compute episode advantages $\Aep_t$ using prompt-group statistics \Comment{Eq.~\plaineqref{eq:episode}}
\State group matching anchor observations and compute $\Ast_t$, with singleton fallback \Comment{Eq.~\plaineqref{eq:step}}
\For{each trajectory $\tau$}
  \State compute sampled-response NLL $H_t$ and criticality scores $c_t$ \Comment{Eqs.~\plaineqref{eq:nll}--\plaineqref{eq:crit}}
  \State $\lambda_t\gets\min\{\lambda_0(1+\alpha_\lambda c_t),1\}$ for every step
  \State $\Afin_t\gets(1-\lambda_t)\Aep_t+\lambda_t\Ast_t$ for every step
\EndFor
\State \Return $\{\Afin_t\}$ for the policy update in Eq.~\plaineqref{eq:surrogate}
\end{algorithmic}
\end{algorithm}

\subsection{Why adapt the credit granularity?}
\label{sec:theory}
We study the squared error of mixed credit estimates relative to the action advantage under a fixed rollout policy $\pi$. Let $U=\phi(X)$ be the environment action parsed from response $X$, and let $Q^\pi(s,u)$ be its expected discounted return from full state $s$ under $\pi$. The action advantage is $A^\pi(s,u)=Q^\pi(s,u)-\E_{U\mid s}Q^\pi(s,U)$.

We use the normalized target $Y=A^\pi/\sigma_A$, where $\sigma_A^2=\E[(A^\pi)^2]>0$ is computed over decision steps sampled from $\pi$. The inputs $E=\Aep$ and $S=\Ast$ retain their implemented normalizations. Their errors relative to $Y$ include grouping effects, scale mismatch, and shared rollout noise. Full statements and proofs are in Appendix~\ref{app:theory}--\ref{app:proofs}.

\paragraph{Conditional mixing balances the two estimation errors.}
At a given score $c$, let $M_E=\E[(E-Y)^2\mid c]$ and $M_S=\E[(S-Y)^2\mid c]$ be the two mean squared errors, and $\chi=\E[(E-Y)(S-Y)\mid c]$ their error cross-moment. With finite second moments and $q=M_E+M_S-2\chi=\E[(S-E)^2\mid c]>0$, the optimal conditional weight is
\begin{equation}
f^\star(c)=\Pi_{[0,1]}\!\left(\frac{M_E(c)-\chi(c)}{M_E(c)+M_S(c)-2\chi(c)}\right),
\label{eq:main-oracle}
\end{equation}
where $\Pi_{[0,1]}$ projects onto the unit interval. This oracle improves on the best constant mixture whenever its unique optimum varies over scores with $q>0$ (Proposition~\ref{prop:mse}). For an implemented gate $f$, let $u(c)=(M_E-\chi)/q$, $\lambda_{\rm fix}^\star=\Pi_{[0,1]}\{\E[M_E-\chi]/\E[q]\}$, and $\mathcal R(f)=\E[((1-f(c))E+f(c)S-Y)^2]$. Proposition~\ref{prop:mse} gives
\begin{equation}
\mathcal R(f)-\mathcal R(\lambda_{\rm fix}^\star)
=\E\!\left[q(c)\left\{[f(c)-u(c)]^2-[\lambda_{\rm fix}^\star-u(c)]^2\right\}\right].
\label{eq:main-riskgap}
\end{equation}
Terms with $q=0$ contribute zero. The score is useful when its gate tracks the conditional error trade-off more closely than the best constant mixture. Here, closeness is measured by expected squared distance to $u(c)$, weighted by $q(c)$, so weight assignment matters more where the estimators disagree. This motivates testing both adaptive weights and their assignment to steps.

\paragraph{Action-value variation can favor finer credit.}
Consider the coarse-to-fine model
\begin{equation}
E=(1-\eta)Y+\xi_E,\qquad S=Y+\xi_S,\qquad 0\le\eta\le1,
\label{eq:main-cfmodel}
\end{equation}
with $\E[Y\xi_E\mid s]=\E[Y\xi_S\mid s]=0$. The episode-side signal attenuates action-specific variation by $\eta$, while the step-side signal retains it. Define $v_E=\E[\xi_E^2\mid s]$, $v_S=\E[\xi_S^2\mid s]$, and $\chi=\E[\xi_E\xi_S\mid s]$. The action-value heterogeneity $C(s)=\Var_{U\mid s}Q^\pi(s,U)$ measures how much the expected outcome changes across available actions. Proposition~\ref{prop:coarsefine} yields
\begin{equation}
\lambda_{\rm cf}^\star(C)=\Pi_{[0,1]}\!\left(
\frac{\eta^2 C/\sigma_A^2+v_E-\chi}{\eta^2 C/\sigma_A^2+v_E+v_S-2\chi}\right).
\label{eq:main-cfweight}
\end{equation}
For a positive denominator, holding $\eta$, the residual moments, and $\sigma_A$ fixed, this weight is nondecreasing in $C$ when $v_S>\chi$, and strictly increasing in the interior when $\eta>0$. Greater action-value variation then increases the cost of coarse credit and favors the finer estimator. Appendix~\ref{app:theory} gives a shared-noise interpretation of the residual condition.

\paragraph{Expected NLL links the score to this regime.}
Let $h(s)=\E[-L^{-1}\log\pi(X\mid s)\mid s]$, where $L$ is the full response length in tokens. Suppose the distribution over parsed environment actions has finite support of size $K_s\ge2$, $L\ge L_{\min}>0$, and linguistic variation among responses executing the same action satisfies $H(X\mid U,s)\le B_s$. The entropy chain rule gives $H(U\mid s)\ge L_{\min}h(s)-B_s$. Thus expected NLL exceeding the within-action variation implies uncertainty over distinct environment actions.

For independent actions $U,U'$, define diversity as $D(s)=\Pr(U\ne U'\mid s)$. Proposition~\ref{prop:nllcrit} constructs a lower bound $\underline D_s$ such that
\begin{equation}
D(s)\ge\underline D_s(h(s)).
\label{eq:main-diversity}
\end{equation}
For fixed $K_s,B_s,L_{\min}$, this bound is nondecreasing in $h$ and positive when $L_{\min}h>B_s$. If distinct actions also have mean squared value separation at least $\kappa_0>0$, then
\begin{equation}
C(s)=\tfrac12D(s)\E[(Q^\pi(s,U)-Q^\pi(s,U'))^2\mid U\ne U',s]
\ \ge\ \tfrac{\kappa_0}{2}\underline D_s(h(s)).
\label{eq:main-valuebound}
\end{equation}
Composing this bound with Eq.~\ref{eq:main-cfweight} gives a nondecreasing lower bound on the preferred step-level weight at fixed regularity and residual-noise parameters. This establishes the conditional route from expected NLL to action diversity, action-value heterogeneity, and finer credit.

GACA follows this direction using the NLL of one sampled response per decision. Trajectory normalization preserves the ordering of NLL within each trajectory, and the gate increases with the normalized score. The ablations below evaluate the resulting allocation and modulation strength. Appendix~\ref{app:theory} gives the exact diversity bound, and Appendix~\ref{app:implemented-score} relates expected NLL to the implemented score.

\section{Experiments}
\label{sec:experiments}
Our evaluation measures attained task success on two agent benchmarks, progress during training, and per-update computational cost. Matched ablations examine stepwise credit allocation, followed by diagnostics of the NLL signal.

\subsection{Experimental setup}
\label{sec:setup}
\paragraph{Tasks and training.}
We use the verl-agent stack~\citep{feng2026group,sheng2025hybridflow} with Qwen2.5-Instruct 1.5B and 7B. ALFWorld involves household tasks with a maximum of $50$ steps; WebShop involves product search and purchase within $15$ steps. GACA samples groups of $8$ trajectories and applies an invalid-action penalty of $0.1$. For ALFWorld 1.5B, we use a two-step history window, rollout temperature $1$, and $16$ task groups per update. We train this model for $300$ parameter updates. We report its update-$150$, best, and final evaluations to distinguish performance at a shorter budget from gains under continued training. Appendix~\ref{app:hyper} gives training settings and benchmark details.

\paragraph{Evaluation and checkpoint selection.}
For ALFWorld, we follow HGPO's evaluation-seed convention~\citep{he2026hierarchy}: each selected checkpoint receives three stochastic evaluations, summarized by their mean and population SD. Each backbone uses one trained model. The 1.5B run uses the in-distribution evaluation split, $128$ tasks, and temperature $0.4$. It evaluates every $5$ updates and selects the update with the highest mean validation success. The selected ALFWorld updates are $275$ (1.5B) and $265$ (7B). Figure~\ref{fig:training-gate} also reports progress at fixed updates. Appendix~\ref{app:eval-protocol} gives the seed values, selection rule, and published-reference protocols.

\paragraph{Reference methods and internal controls.}
The benchmark comparison focuses on attained task success, using selected GACA checkpoints and the strongest published configurations of the listed reference methods.
GRPO and GiGPO results are quoted from \citet{feng2026group}; HGPO and its accompanying GiGPO references are quoted from \citet{he2026hierarchy}. Table~\ref{tab:main} retains the $K=4$ HGPO configurations, the stronger reported setting in each framework version, and distinguishes their WebShop horizons. The complete $K=2,4$ comparison appears in Table~\ref{tab:published-full}. Internal ALFWorld 1.5B ablations retain the same advantage definitions, normalization, and singleton fallback under a common training and selection protocol. The fixed-weight controls retain GACA's convex-mixture scale; Appendix~\ref{app:impl} distinguishes this scale from the original GiGPO update.

\subsection{Overall performance}
\label{sec:main}
Table~\ref{tab:main} reports success rates across four backbone--benchmark settings. GACA has the highest reported mean among the listed methods in each setting, reaching $96.88\%$ and $98.18\%$ on ALFWorld and $80.99\%$ and $83.59\%$ on WebShop at 1.5B and 7B, respectively. Relative to updated-framework HGPO ($K=4$), the differences are $2.03$ and $2.22$ percentage points on ALFWorld and $2.87$ and $4.30$ points on WebShop. The margins over the original GiGPO reference range from $7.4$ to $13.6$ points.

\begin{table}[H]
\caption{Success rates (\%, mean $\pm$ SD). External rows are quoted from the indicated sources. HGPO's ALFWorld entries are in-distribution results; $K$ is the history-window length. The complete history-window comparison is in Table~\ref{tab:published-full}. WebShop horizons and framework versions are stated per block. GACA's ALFWorld deviations are across three evaluation seeds at one selected checkpoint.}
\label{tab:main}
\centering\small
\setlength{\tabcolsep}{6pt}
\renewcommand{\arraystretch}{1.12}
\begin{tabular}{@{}lcccc@{}}
\toprule
\multirow{2}{*}{\textbf{Method}} & \multicolumn{2}{c}{\textbf{ALFWorld}} & \multicolumn{2}{c}{\textbf{WebShop}} \\
\cmidrule(lr){2-3}\cmidrule(lr){4-5}
& \textbf{1.5B} & \textbf{7B} & \textbf{1.5B} & \textbf{7B} \\
\midrule
\multicolumn{5}{@{}l}{\textit{\citet{feng2026group}; WebShop: 15 steps}}\\
GRPO & $72.8$\,\sd{3.6} & $77.6$\,\sd{5.2} & $56.8$\,\sd{3.8} & $66.1$\,\sd{3.7} \\
GiGPO & $86.7$\,\sd{1.7} & $90.8$\,\sd{1.3} & $67.4$\,\sd{4.5} & $75.2$\,\sd{3.8} \\
\midrule
\multicolumn{5}{@{}l}{\textit{\citet{he2026hierarchy}, Table 7: earlier framework; WebShop: 15 steps}}\\
HGPO ($K=4$) & $92.45$\,\sd{0.81} & $94.79$\,\sd{0.90} & $75.91$\,\sd{1.19} & $77.21$\,\sd{0.22} \\
\midrule
\multicolumn{5}{@{}l}{\textit{\citet{he2026hierarchy}, Table 1: updated framework; WebShop: 30 steps}}\\
GiGPO ($K=4$) & $93.29$\,\sd{1.07} & $95.63$\,\sd{0.49} & $73.24$\,\sd{2.25} & $77.79$\,\sd{0.92} \\
HGPO ($K=4$) & $94.85$\,\sd{0.92} & $95.96$\,\sd{0.49} & $78.12$\,\sd{2.06} & $79.29$\,\sd{1.17} \\
\midrule
\multicolumn{5}{@{}l}{\textit{This work; WebShop: 15 steps}}\\
\rowcolor{gacashade}
\textbf{GACA} & $\mathbf{96.88}$\,\sd{0.64} & $\mathbf{98.18}$\,\sd{0.37} & $\mathbf{80.99}$\,\sd{2.24} & $\mathbf{83.59}$\,\sd{3.38} \\
\bottomrule
\end{tabular}
\end{table}

\paragraph{Comparison with history-aware grouping.}
HGPO improves the local advantage by grouping steps at different history depths. GACA retains GiGPO's anchor grouping and changes the relative weight of episode and step credit. The two approaches act on different parts of credit assignment: construction of the local estimate and its contribution to the final update. The earlier-framework HGPO rows provide a $15$-step WebShop reference, while the updated rows retain the stronger published results with their $30$-step horizon. These comparisons preserve each source's training and evaluation protocol; the internal ablations below isolate GACA's allocation choices.

\paragraph{Performance over training.}
ALFWorld 1.5B reaches $92.19\pm3.88\%$ by update $150$. Continued training raises success by $4.69$ percentage points to $96.88\pm0.64\%$ at update $275$, with $96.09\pm1.10\%$ retained at the final update $300$ (Figure~\ref{fig:training-gate}). This trajectory shows strong performance within the first $150$ updates and further gains with continued optimization. Table~\ref{tab:alfworld-eval-seeds} reports the individual evaluation scores at these updates.

\begin{figure}[t]
\centering
\includegraphics[width=0.60\textwidth]{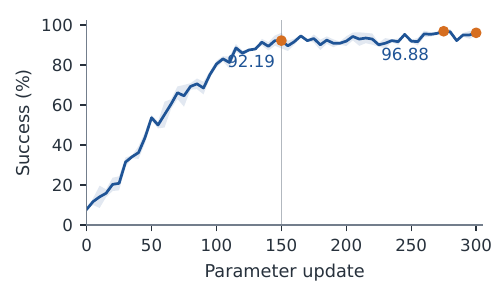}
\caption{ALFWorld 1.5B validation success over $300$ parameter updates. The curve and band show the mean and population SD across three evaluation seeds for one training run. Markers identify update $150$, the selected update $275$, and the final update $300$.}
\label{fig:training-gate}
\end{figure}

\paragraph{Computational cost.}
GACA reuses log-probabilities already computed for the policy update and adds scalar NLL, score, and mixing operations. These operations take $0.09$ seconds per iteration in the ALFWorld 1.5B timing breakdown, against approximately $372$ seconds for the complete iteration. The reported peak memory remains $62.4$\,GB. These timings average $20$ iterations on the same hardware and quantify the incremental cost of the adaptive rule. Appendix~\ref{app:cost} gives the full breakdown.

\subsection{Effect of adaptive credit allocation}
\label{sec:ablations}
We use ALFWorld 1.5B to test three aspects of credit allocation: combining episode and step signals, adapting their mixture beyond a constant weight, and assigning the resulting weights to individual decisions. Figure~\ref{fig:results} presents the fixed-weight sweep and random assignment control. The historical step-weighted mean is $\bar\lambda_{\rm train}\approx0.748$, motivating fixed $\lambda=0.75$ as a near-mean comparator. Appendix~\ref{app:suppablation} gives the complete numerical grid.

\paragraph{Adaptive versus fixed mixing.}
With adaptation disabled ($\alpha_\lambda=0$), we vary the constant weight over $\{0,0.25,0.5,0.75,1\}$. The episode-only and step-only endpoints reach $85.94\%$ and $84.60\%$, while every tested interior mixture performs better than both endpoints. Combining the two granularities therefore helps even with a fixed weight. The best tested constant, $\lambda=0.5$, reaches $93.05\%$; fixed $\lambda=0.75$, close to GACA's mean weight, reaches $92.41\%$. With the same episode- and step-level estimators, GACA reaches $96.88\%$, exceeding these controls by $3.83$ and $4.47$ points, respectively. The results support adapting the relative contribution of the two credit signals across decisions.

\begin{figure}[htbp]
\centering
\includegraphics[width=\textwidth]{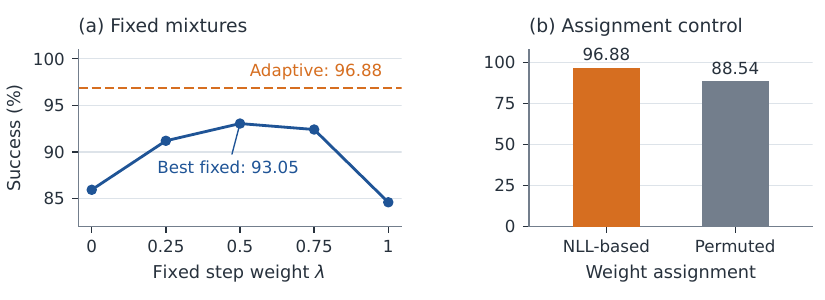}
\caption{Matched ALFWorld 1.5B ablations, using means over three evaluation seeds. \textbf{Left:} adaptive mixing exceeds the best tested fixed weight by $3.83$ points. \textbf{Right:} permuting NLL scores within each trajectory preserves its weight multiset but reduces success by $8.34$ points. All values also appear in Table~\ref{tab:ablation}.}
\label{fig:results}
\end{figure}

\paragraph{Does the assignment to steps matter?}
Randomly permuting criticality scores within each trajectory reduces success from $96.88\%$ to $88.54\%$. This control retains the trajectory's collection of gate values and changes their pairing with the episode and step advantages. The $8.34$-point drop shows that performance depends on where the weights are assigned. Together with the fixed-weight sweep, it supports using the score to allocate the two credit signals at individual decisions.

The realized step weights span $0.623$--$0.900$ at the 10th--90th percentiles in frozen-policy rollouts; Appendix~\ref{app:diagnostic-weights} gives their distribution and the training mean.

\paragraph{Modulation strength.}
All three tested nonzero modulation strengths have higher mean success than the modulation-off control, with $\alpha_\lambda=0.5$ giving the highest mean. Appendix~\ref{app:suppablation} reports the sensitivity results.

\subsection{Within-trajectory analysis of the NLL score}
\label{sec:diagnostics}
We examine whether the score used by GACA identifies decisions with greater action diversity within a trajectory. Using the frozen ALFWorld 1.5B final@300 checkpoint, we select $32$ training tasks stratified by task type and collect $8$ trajectories per task. From one uniformly sampled trajectory per task, we draw two decisions uniformly without replacement, giving $64$ states. Each pair is ordered by its original rollout criticality $c_t$. At each state, $32$ independently sampled responses estimate action disagreement $D$, the probability that two responses execute different parsed actions.

Higher-score states have mean disagreement $0.1789$, compared with $0.1041$ for lower-score states (Figure~\ref{fig:mechanism}). The mean paired difference is $0.0747$, with a $95\%$ task-bootstrap interval of $[0.0248,0.1334]$.

Higher-score decisions thus exhibit greater action diversity on average within the same trajectory. Together, the allocation controls and paired analysis show that the assignment of weights affects training performance and that higher weights are associated with more action-diverse decisions. Appendix~\ref{app:diagnostic-diversity} gives the sampling and estimation details.

\begin{figure}[htbp]
\centering
\includegraphics[width=\textwidth]{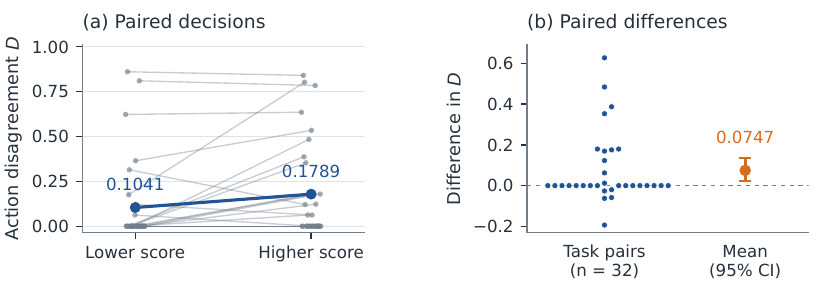}
\caption{Within-trajectory action diversity for the frozen ALFWorld 1.5B policy. \textbf{Left:} paired states ordered by original $c_t$; blue markers show group means. \textbf{Right:} all $32$ higher-minus-lower differences (blue), with their mean and $95\%$ task-bootstrap interval (orange). Horizontal offsets separate overlapping observations. Diversity is estimated from fresh responses.}
\label{fig:mechanism}
\end{figure}
\FloatBarrier

\section{Conclusion}
\label{sec:conclusion}
GACA uses rollout NLL to adapt the mixture of episode- and step-level credit during policy optimization, without additional training rollouts or model evaluations. Conditional risk analysis characterizes score-dependent mixing and conditions favoring finer credit. Across ALFWorld and WebShop with two backbone sizes, GACA achieves the highest reported mean success rates among the compared methods. In matched ALFWorld ablations, GACA improves over the tested fixed weights and within-trajectory random assignment. Together, these results support step-dependent credit mixing as a practical design choice for critic-free agent training.

\label{sec:main-end}

\section*{AI Use Statement}
Generative AI tools assisted with manuscript editing, theoretical exposition, figure preparation, and diagnostic experiment design. The authors are responsible for verifying the arguments, implementation descriptions, references, and reported results.

\section*{Reproducibility Statement}
Appendix~\ref{app:theory}--\ref{app:proofs} contain the theoretical statements and proofs. Appendix~\ref{app:hyper}--\ref{app:impl} specify training settings, advantage normalization, singleton fallback, and criticality construction. Appendix~\ref{app:suppablation} gives the full numerical ablation grid, and Appendix~\ref{app:diagnostic-protocol} defines the diagnostic estimators. Section~\ref{sec:setup} records the checkpoint-selection protocol and the source of external baseline results.

\bibliography{user_references}
\bibliographystyle{iclr2027_conference}
\clearpage
\appendix
\section{Extended Related Work}
\label{app:related}

\textbf{Critic-free RL for LLMs and agents.} Replacing the learned value function with a group baseline has become a common recipe for RL on language models, from GRPO~\citep{shao2024deepseekmath,guo2025deepseek} and REINFORCE-style variants~\citep{ahmadian2024back} to refinements of clipping and normalization~\citep{yu2026dapo}. Carrying this from single-turn reasoning to multi-turn agents required new infrastructure and recipes~\citep{sheng2025hybridflow,wang2025ragen}. GiGPO~\citep{feng2026group} is the direct methodological precursor from which our reported implementation inherits anchor-observation step grouping. Subsequent work has exposed limitations of observation-only step groups: HGPO uses hierarchical history-consistent groups to address context inconsistency~\citep{he2026hierarchy}, while BiPACE replaces surface observation hashing with behavior-aware grouping and action-conditioned peer baselines~\citep{wang2026bipace}. ProGPO addresses a complementary long-horizon credit trap by introducing progress-conditioned relative advantages in failure-dominated groups~\citep{yang2026progress}. These methods modify the construction of local credit estimates. GACA addresses a complementary choice: the relative weight assigned to local and episode-level credit at each step. The mixing analysis accommodates alternative step estimators; the reported experiments use GiGPO-style anchor grouping.

\textbf{Credit assignment and reward redistribution.} Assigning delayed outcomes to individual decisions is a central problem in reinforcement learning. Approaches include return decomposition~\citep{arjona2019rudder}, hindsight relabeling~\citep{andrychowicz2017hindsight}, generalized advantage estimation~\citep{schulman2015high}, and potential-based shaping~\citep{ng1999policy}. For LLMs, process reward models supply dense step-level supervision where intermediate correctness is checkable~\citep{lightman2024let,wang2024math,uesato2022solving}, and VinePPO replaces the critic with Monte-Carlo value estimates from extra rollouts branched at intermediate states~\citep{kazemnejad2024vineppo}. These approaches use additional process supervision or intermediate-state rollouts. GACA reuses rollout log-probabilities to condition the mixture of existing credit estimates. GAE trades bias and variance through a mixture of temporal estimates; GACA instead mixes two critic-free group-relative estimates using a per-step score. Proposition~\ref{prop:mse} characterizes score-dependent mixing relative to local credit, while Proposition~\ref{prop:coarsefine} isolates a regime in which finer credit becomes increasingly valuable as local action values separate.

\textbf{Uncertainty-guided advantage modulation.} Recent work also uses surprisal to reshape RL updates. AEM~\citep{zhao2026aem} analyzes response-level entropy dynamics in multi-turn agents and rescales a base response-level advantage using a response-level uncertainty proxy. STARE~\citep{luo2026stare} identifies high-surprisal tokens and selectively reweights their effective advantages to stabilize policy entropy. These methods rescale a pre-existing credit signal. GACA uses the score to construct a mixture of episode- and step-level credit estimates.

\textbf{Entropy and the value of uncertain decisions.} Entropy has long shaped exploration and stability in RL~\citep{haarnoja2018soft}, and recent analyses of RL for LLM reasoning show that a small, high-entropy minority of tokens drives most of the policy improvement while entropy collapse degrades learning~\citep{wang2026beyond,cui2025entropy}. GACA uses uncertainty at the environment-step level to decide how much to rely on the two credit granularities.

\section{Detailed Theory Statements}
\label{app:theory}

This appendix gives the full formulations underlying the three results summarized in Section~\ref{sec:theory}. Throughout, $\pi=\pi_{\theta_{\mathrm{old}}}$ is fixed.

\paragraph{Local-credit target and conditional mixing.}
Let $X$ denote the full sampled response, including reasoning and the action command, and let $U=\phi(X)$ be the corresponding environment action after parsing. At a full pre-action state $s$, define
\[
Q^\pi(s,u)=\E[\hat g_t\mid s_t=s,U_t=u],\qquad
V^\pi(s)=\E_{U\sim\pi(\cdot\mid s)}[Q^\pi(s,U)],
\]
and the action advantage $A^\pi(s,u)=Q^\pi(s,u)-V^\pi(s)$. Let the old-policy step-row population be the reference measure and assume $0<\sigma_A^2<\infty$ and finite second moments for $E$ and $S$. Define
\[
\sigma_A^2=\E\!\left[(A^\pi(s_t,U_t))^2\right]>0,
\]
and use the dimensionless local-credit target
\begin{equation}
Y_t=\frac{A^\pi(s_t,U_t)}{\sigma_A}.
\label{eq:local-target}
\end{equation}
The target scale is fixed by the old-policy population. Differences between this normalization and the implemented group normalizations enter the estimator error. Let $E=\Aep$ and $S=\Ast$ denote the two \emph{implemented} inputs to equation~\plaineqref{eq:final}, after their actual group normalization and any singleton fallback. Normalization mismatch, grouping bias, and shared rollout noise are included in
\begin{equation}
\begin{gathered}
e_E=E-Y,\qquad e_S=S-Y,\\
M_E(c)=\E[e_E^2\mid c],\qquad
M_S(c)=\E[e_S^2\mid c],\qquad
\chi(c)=\E[e_Ee_S\mid c],\\
q(c)=M_E(c)+M_S(c)-2\chi(c)
     =\E[(S-E)^2\mid c].
\end{gathered}
\label{eq:conditional-moments}
\end{equation}
The score $c$ may be statistically dependent on the sampled action, the trajectory, and both estimator errors.

\begin{proposition}[Optimal conditional mixing relative to local credit]
\label{prop:mse}
For $\lambda\in[0,1]$, the conditional squared-credit risk of $(1-\lambda)E+\lambda S$ is
\begin{equation}
R_c(\lambda)
=(1-\lambda)^2M_E(c)+\lambda^2M_S(c)
+2\lambda(1-\lambda)\chi(c).
\label{eq:conditional-risk}
\end{equation}
On $\{q(c)>0\}$, define
\[
u(c)=\frac{M_E(c)-\chi(c)}{q(c)}.
\]
The risk-minimizing score-dependent weight is
\begin{equation}
f^\star(c)=\Pi_{[0,1]}\!\big(u(c)\big).
\label{eq:lamstar}
\end{equation}
If $\E[q(c)]>0$, the best constant weight is
\begin{equation}
\lambda_{\mathrm{fix}}^\star
=\Pi_{[0,1]}\!\left(
\frac{\E[M_E(c)-\chi(c)]}{\E[q(c)]}
\right).
\label{eq:best-fixed}
\end{equation}
Moreover, every measurable gate $f(c)\in[0,1]$ satisfies
\begin{equation}
\mathcal R(f)-\mathcal R(\lambda_{\mathrm{fix}}^\star)
=
\E\!\left[
q(c)\left(
[f(c)-u(c)]^2
-
[\lambda_{\mathrm{fix}}^\star-u(c)]^2
\right)
\right],
\label{eq:actual-vs-fixed}
\end{equation}
where $\mathcal R(f)=\E[((1-f(c))E+f(c)S-Y)^2]$. If $q(c)=0$, then $S=E$ conditionally almost surely and the weight is irrelevant. The oracle satisfies $\mathcal R(f^\star)\le\min_{\lambda\in[0,1]}\mathcal R(\lambda)$. For $\E[q]>0$, the inequality is strict exactly when no constant agrees with $f^\star(c)$ almost surely on $\{q(c)>0\}$.
\end{proposition}

\paragraph{Why local action-value heterogeneity can favor finer credit.}
To connect the generic risk to credit granularity, consider the following explanatory special case. At a fixed full state $s$, let the coarse episode-side signal attenuate action-specific credit while the fine step-side signal retain it:
\begin{equation}
E=(1-\eta)Y+\xi_E,\qquad
S=Y+\xi_S,\qquad 0\le\eta\le1.
\label{eq:coarse-fine-model}
\end{equation}
Assume
\[
\E[Y\xi_E\mid s]=\E[Y\xi_S\mid s]=0,
\]
and write
\[
v_E=\E[\xi_E^2\mid s],\qquad
v_S=\E[\xi_S^2\mid s],\qquad
\chi=\E[\xi_E\xi_S\mid s].
\]
Because $U\sim\pi(\cdot\mid s)$,
\begin{equation}
\E[Y\mid s]=0,\qquad
\E[Y^2\mid s]
=
\frac{C(s)}{\sigma_A^2},
\qquad
C(s):=\Var_{U\sim\pi(\cdot\mid s)}[Q^\pi(s,U)].
\label{eq:C-def}
\end{equation}
Thus $C(s)$ measures how much the expected continuation value changes across actions available at the same full state.

\begin{proposition}[Coarse-to-fine comparative statics]
\label{prop:coarsefine}
Under equation~\plaineqref{eq:coarse-fine-model}, the unprojected risk-minimizing fine-credit weight is
\begin{equation}
u_{\mathrm{cf}}(C)
=
\frac{\eta^2 C/\sigma_A^2+v_E-\chi}
     {\eta^2 C/\sigma_A^2+v_E+v_S-2\chi},
\label{eq:coarse-fine-weight}
\end{equation}
whenever the denominator is positive, and the constrained optimum is
$\lambda_{\mathrm{cf}}^\star=\Pi_{[0,1]}(u_{\mathrm{cf}})$. Holding $\eta,v_E,v_S,\chi,$ and $\sigma_A$ fixed,
\begin{equation}
\frac{\partial u_{\mathrm{cf}}}{\partial C}
=
\frac{\eta^2(v_S-\chi)}
{\sigma_A^2\big(\eta^2C/\sigma_A^2+v_E+v_S-2\chi\big)^2}.
\label{eq:coarse-fine-derivative}
\end{equation}
Hence, if $v_S>\chi$, the preferred fine-credit weight is nondecreasing in local action-value heterogeneity $C$; it is strictly increasing in the interior whenever $\eta>0$.
\end{proposition}

The condition $v_S>\chi$ has a simple shared-noise interpretation. If
\[
\xi_E=\zeta+\varepsilon_E,\qquad
\xi_S=\zeta+\varepsilon_S,
\]
with conditionally zero-mean components, the shared component $\zeta$ uncorrelated with the two residuals, and the residuals mutually uncorrelated, then
\[
v_S-\chi=\Var(\varepsilon_S)\ge0,
\]
with strict inequality whenever the fine estimator has nonzero idiosyncratic residual noise. Proposition~\ref{prop:coarsefine} describes variation in $C$ within this model at fixed residual moments. Applying it across states also requires accounting for changes in those moments.

\paragraph{Expected NLL as a structural certificate.}
We next relate expected response surprisal to variation in the executed action. Condition on the \emph{full} pre-action state $s$, including the interaction history. Let $X\sim\pi(\cdot\mid s)$ be the full response of token length $L$, and let $U=\phi(X)$ be its environment-executed parsed action identity. Define
\begin{equation}
\ell=-L^{-1}\log\pi(X\mid s),
\qquad
h(s)=\E[\ell\mid s].
\label{eq:expected-nll}
\end{equation}
Let the canonical action distribution at $s$ have finite support size $K_s\ge2$, let $L\ge L_{\min}>0$, and suppose the within-action textual entropy satisfies
\[
H(X\mid U,s)\le B_s.
\]
The finite-support requirement is natural under a finite tokenizer vocabulary and bounded response length; the substantive quantity is $B_s$, which separates environment-action uncertainty from linguistic variation among texts that execute the same action.

Define
\begin{equation}
\begin{gathered}
\underline H_s(h)
=
\min\{[L_{\min}h-B_s]_+,\log K_s\},\\
\underline D_s(h)
=
1-F_{K_s}^{-1}\!\big(\underline H_s(h)\big),\\
F_K(p)=h_b(p)+(1-p)\log(K-1),
\qquad
h_b(p)=-p\log p-(1-p)\log(1-p).
\end{gathered}
\label{eq:diversity-lb}
\end{equation}

\begin{proposition}[Expected NLL certifies canonical-action diversity]
\label{prop:nllcrit}
Let $U,U'$ be independent draws from the induced canonical-action law at $s$, and define
\[
D(s)=\Pr(U\ne U'\mid s).
\]
Then
\begin{equation}
D(s)\ge \underline D_s(h(s)).
\label{eq:action-diversity}
\end{equation}
For fixed $K_s,B_s,$ and $L_{\min}$, the lower bound is nondecreasing in expected NLL $h(s)$ and is positive whenever $L_{\min}h(s)>B_s$.
\end{proposition}

\begin{corollary}[From action diversity to local value heterogeneity]
\label{cor:nll-value}
Suppose in addition that, whenever $D(s)>0$,
\begin{equation}
\kappa_s:=
\E\!\left[
\big(Q^\pi(s,U)-Q^\pi(s,U')\big)^2
\mid U\ne U',s
\right]
\ge\kappa_0>0.
\label{eq:value-separation}
\end{equation}
Then
\begin{equation}
C(s)
\ge
\underline C_s(h(s))
:=
\frac{\kappa_0}{2}\,
\underline D_s(h(s)).
\label{eq:value-variation-lb}
\end{equation}
\end{corollary}

The two conditions serve distinct roles: $B_s$ bounds linguistic variation among responses executing the same action, while $\kappa_0$ links distinct actions to differences in expected outcome.

\paragraph{From NLL to a lower envelope on preferred granularity.}
Combining the two preceding pieces yields the population-level directional link used to motivate GACA.

\begin{corollary}[NLL-induced lower envelope on fine-credit preference]
\label{cor:nll-mixing}
Consider the coarse-to-fine family of Proposition~\ref{prop:coarsefine} with fixed $\eta,v_E,v_S,\chi,$ and $\sigma_A$, and suppose $v_S>\chi$ and equation~\plaineqref{eq:value-separation} holds. Define
\begin{equation}
\underline\lambda(h)
=
\Pi_{[0,1]}
\left(
\frac{
\eta^2\underline C_s(h)/\sigma_A^2+v_E-\chi
}{
\eta^2\underline C_s(h)/\sigma_A^2+v_E+v_S-2\chi
}
\right).
\label{eq:lambda-lower-envelope}
\end{equation}
Then
\begin{equation}
\lambda_{\mathrm{cf}}^\star(s)
\ge
\underline\lambda(h(s)),
\label{eq:lambda-envelope-bound}
\end{equation}
and $\underline\lambda(h)$ is nondecreasing in $h$ for fixed regularity and residual-noise parameters.
\end{corollary}

\paragraph{Alignment with the implemented trajectory-relative gate.}
GACA uses one sampled-response NLL per decision and sets $c_t=H_t/\overline H_\tau$ after rollout completion. For $\overline H_\tau>0$, normalization preserves NLL ordering within the trajectory, and the gate in equation~\plaineqref{eq:lambda} assigns at least as much step-level weight to higher-scoring responses. The population analysis motivates this increasing direction; the gate floor and modulation strength are empirical design choices.

\section{Proofs and Technical Remarks}
\label{app:proofs}

All results are fixed-policy statements under $\pi=\pi_{\theta_{\mathrm{old}}}$. The implemented inputs $E=\Aep$ and $S=\Ast$ may share rollout randomness, may be biased relative to the local-credit target, and may use random group normalizations; no independence between $E$, $S$, and the gate score is assumed in Proposition~\ref{prop:mse}. Entropies use natural logarithms.

\subsection{Proof of Proposition~\ref{prop:mse}}
Expanding equation~\plaineqref{eq:conditional-risk},
\[
R_c(\lambda)
=
M_E(c)
-2\lambda[M_E(c)-\chi(c)]
+\lambda^2 q(c).
\]
For $q(c)>0$, completing the square gives
\[
R_c(\lambda)
=
M_E(c)
-\frac{[M_E(c)-\chi(c)]^2}{q(c)}
+
q(c)[\lambda-u(c)]^2.
\]
Projection onto $[0,1]$ therefore gives equation~\plaineqref{eq:lamstar}. If $q(c)=0$, the identity
$q(c)=\E[(S-E)^2\mid c]$ implies $S=E$ conditionally almost surely. Taking expectation of the quadratic before minimizing gives equation~\plaineqref{eq:best-fixed}. Subtracting the completed-square forms gives equation~\plaineqref{eq:actual-vs-fixed}. Conditional Cauchy--Schwarz implies
$[M_E-\chi]^2\le qM_E$, so the weighted square terms are integrable whenever the stated second moments are finite.

Finally, projection minimizes the conditional quadratic at every score. On $\{q(c)>0\}$ the minimizer is unique; hence equality with a constant requires that constant to agree with $f^\star(c)$ almost surely there. Contributions from $\{q(c)=0\}$ vanish. \hfill$\square$

\subsection{Proof of Proposition~\ref{prop:coarsefine}}
Under equation~\plaineqref{eq:coarse-fine-model},
\[
e_E=-\eta Y+\xi_E,\qquad
e_S=\xi_S.
\]
The orthogonality conditions and equation~\plaineqref{eq:C-def} give
\[
M_E(s)=\eta^2\frac{C(s)}{\sigma_A^2}+v_E,
\qquad
M_S(s)=v_S,
\qquad
\chi(s)=\chi.
\]
Substitution into the unprojected optimum
$u=(M_E-\chi)/(M_E+M_S-2\chi)$
yields equation~\plaineqref{eq:coarse-fine-weight}. Differentiating with respect to $C$ while holding $\eta,v_E,v_S,\chi,$ and $\sigma_A$ fixed gives equation~\plaineqref{eq:coarse-fine-derivative}. Projection onto a closed interval is monotone, so the constrained optimum inherits the nondecreasing dependence on $C$. Strict increase holds wherever the unprojected optimum lies in the interior and $\eta>0$, $v_S>\chi$. \hfill$\square$

\paragraph{Shared-noise sufficient condition.}
If
$\xi_E=\zeta+\varepsilon_E$ and
$\xi_S=\zeta+\varepsilon_S$
with conditionally zero-mean components and all pairwise cross-covariances among $\zeta,\varepsilon_E,\varepsilon_S$ equal to zero, then
\[
v_S=\Var(\zeta)+\Var(\varepsilon_S),
\qquad
\chi=\Var(\zeta),
\]
and therefore
\[
v_S-\chi=\Var(\varepsilon_S)\ge0.
\]
This shared-noise structure satisfies the second-moment condition in Proposition~\ref{prop:coarsefine}.

\subsection{Proof of Proposition~\ref{prop:nllcrit}}
Condition on a full pre-action state $s$. The textual action $X$ follows the logged old policy, so
\[
H(X\mid s)
=
\E[-\log\pi(X\mid s)\mid s]
=
\E[L\ell\mid s].
\]
Because $L\ge L_{\min}$ and $\ell\ge0$,
\[
H(X\mid s)\ge L_{\min}h(s).
\]
Since $U=\phi(X)$ is a deterministic function of $X$,
\[
H(X\mid s)=H(U\mid s)+H(X\mid U,s),
\]
and the bound $H(X\mid U,s)\le B_s$ gives
\begin{equation}
H(U\mid s)\ge L_{\min}h(s)-B_s.
\label{eq:proof-action-entropy}
\end{equation}
Truncating to the feasible range $[0,\log K_s]$ yields
$H(U\mid s)\ge\underline H_s(h(s))$.

Let $p^{\max}_s=\max_u p_s(u)$. For two independent canonical actions,
\[
D(s)=1-\sum_u p_s(u)^2\ge1-p_s^{\max}.
\]
Among distributions on $K_s$ symbols with largest mass $p\ge1/K_s$, entropy is maximized by assigning the remaining mass uniformly, giving
\[
H(U\mid s)\le
F_{K_s}(p_s^{\max}),
\]
where $F_K$ is defined in equation~\plaineqref{eq:diversity-lb}. Since
$F_K'(p)=\log\!\frac{1-p}{p(K-1)}<0$ for $p>1/K$, its inverse is decreasing on $[0,\log K]$. Hence
\[
p_s^{\max}
\le
F_{K_s}^{-1}\!\big(H(U\mid s)\big)
\le
F_{K_s}^{-1}\!\big(\underline H_s(h(s))\big),
\]
which proves
\[
D(s)\ge\underline D_s(h(s)).
\]
Monotonicity in $h$ follows because $\underline H_s(h)$ is nondecreasing and $F_{K_s}^{-1}$ is decreasing. The lower bound is positive exactly when $L_{\min}h(s)>B_s$. \hfill$\square$

\subsection{Proof of Corollary~\ref{cor:nll-value}}
For independent $U,U'$ from the canonical-action law at $s$,
\[
\E[
(Q^\pi(s,U)-Q^\pi(s,U'))^2
\mid s
]
=
2\Var_U[Q^\pi(s,U)\mid s]
=
2C(s).
\]
The squared difference is zero on $U=U'$, so when $D(s)>0$,
\[
2C(s)
=
D(s)\,
\E[
(Q^\pi(s,U)-Q^\pi(s,U'))^2
\mid U\ne U',s
]
=
D(s)\kappa_s.
\]
Using $\kappa_s\ge\kappa_0$ and Proposition~\ref{prop:nllcrit},
\[
C(s)
\ge
\frac{\kappa_0}{2}D(s)
\ge
\frac{\kappa_0}{2}\underline D_s(h(s))
=
\underline C_s(h(s)).
\]
If $D(s)=0$, both sides are zero. \hfill$\square$

\subsection{Proof of Corollary~\ref{cor:nll-mixing}}
Under the fixed coarse-to-fine parameters in Corollary~\ref{cor:nll-mixing}, Proposition~\ref{prop:coarsefine} shows that the projected optimum is a nondecreasing function of $C$ when $v_S>\chi$. Corollary~\ref{cor:nll-value} gives
$C(s)\ge\underline C_s(h(s))$. Applying the nondecreasing mapping to both sides yields equation~\plaineqref{eq:lambda-envelope-bound}. The lower envelope is nondecreasing in $h$ because $\underline C_s(h)$ is nondecreasing and the coarse-to-fine optimum is nondecreasing in its $C$ argument. \hfill$\square$

\subsection{The implemented NLL score}
\label{app:implemented-score}
The population analysis uses $h(s)=\E[H_t\mid s]$, the expected mean per-token NLL of the full response. GACA uses one sampled $H_t$ at each decision and computes $c_t=H_t/\overline H_\tau$ from the completed rollout. The score can therefore depend on the sampled action and subsequent continuation. During policy optimization, scores and advantages are treated as fixed quantities, with no gradients propagated through them. Proposition~\ref{prop:mse} allows their statistical dependence.

\section{Training Configuration and Hyperparameters}
\label{app:hyper}
Table~\ref{tab:benchmarks} describes the two benchmarks. Table~\ref{tab:setup} lists training settings; the archived ALFWorld 1.5B configuration sets $300$ epochs and its log records $300$ parameter updates. Table~\ref{tab:hyper} gives gate defaults and search ranges. Appendix~\ref{app:impl} specifies advantage normalization and singleton handling. Appendix~\ref{app:suppablation} reports the parameter sensitivity results.

\begin{table}[H]
\caption{Benchmarks used in our evaluation. Both use the agent-environment interface of Section~\ref{sec:mdp} with a sparse, episode-level success reward and a small invalid-action penalty.}
\label{tab:benchmarks}
\centering
\small
\setlength{\tabcolsep}{7pt}
\renewcommand{\arraystretch}{1.15}
\begin{tabular}{@{}llccc@{}}
\toprule
\textbf{Benchmark} & \textbf{Setting} & \textbf{Max steps} & \textbf{Eval tasks} & \textbf{Metric} \\
\midrule
ALFWorld & text-based embodied household tasks & 50 & 128 & success rate \\
WebShop  & grounded web shopping and navigation & 15 & 128 & success rate \\
\bottomrule
\end{tabular}
\end{table}

\begin{table}[H]
\caption{ALFWorld validation success (\%) at fixed and selected parameter updates. Evaluation columns enumerate the three sampling seeds at the same checkpoint; the summary uses their population standard deviation. Updates $275$ (1.5B) and $265$ (7B) are selected by validation mean.}
\label{tab:alfworld-eval-seeds}
\centering\small
\setlength{\tabcolsep}{8pt}
\renewcommand{\arraystretch}{1.1}
\begin{tabular}{@{}lrrrrc@{}}
\toprule
Backbone & Update & Eval.\ 1 & Eval.\ 2 & Eval.\ 3 & Mean $\pm$ SD \\
\midrule
1.5B & 150 & 89.844 & 89.063 & 97.656 & $92.19\pm3.88$ \\
1.5B & 275 & 96.875 & 96.094 & 97.656 & $96.88\pm0.64$ \\
1.5B & 300 & 95.313 & 95.313 & 97.656 & $96.09\pm1.10$ \\
7B & 265 & 98.44 & 97.66 & 98.44 & $98.18\pm0.37$ \\
\bottomrule
\end{tabular}
\end{table}

\begin{table}[H]
\caption{Training settings. The archived epoch count is available for ALFWorld 1.5B. NR denotes a setting not specified in the archived configuration. The frozen-policy diagnostics use success reward $10$. Section~\ref{sec:setup} describes evaluation and published references.}
\label{tab:setup}
\centering
\small
\setlength{\tabcolsep}{10pt}
\renewcommand{\arraystretch}{1.1}
\begin{tabular}{@{}lcc@{}}
\toprule
\textbf{Setting} & \textbf{ALFWorld} & \textbf{WebShop} \\
\midrule
Backbone & \multicolumn{2}{c}{Qwen2.5-Instruct, 1.5B and 7B} \\
Group size $G$ & 8 & 8 \\
Configured epochs & 300 (1.5B) & 150 \\
Learning rate & $1\times10^{-6}$ & $1\times10^{-6}$ \\
KL loss (low-variance estimator) & 0.01 & 0.01 \\
Discount $\gamma$ & 0.95 & 0.95 \\
Max environment steps & 50 & 15 \\
Max prompt length & 2048 & 4096 \\
Advantage normalization & mean-and-std-norm & mean-norm \\
Invalid-action penalty magnitude & 0.1 & 0.1 \\
Terminal success reward (training) & NR & 1 \\
\bottomrule
\end{tabular}
\end{table}

\begin{table}[H]
\caption{GACA default hyperparameters and search ranges.}
\label{tab:hyper}
\centering
\small
\renewcommand{\arraystretch}{1.15}
\begin{tabular}{@{}>{\raggedright\arraybackslash}p{0.27\textwidth}>{\raggedright\arraybackslash}p{0.15\textwidth}>{\raggedright\arraybackslash}p{0.48\textwidth}@{}}
\toprule
\textbf{Parameter} & \textbf{Default} & \textbf{Search range} \\
\midrule
criticality signal & normalized NLL & NLL and within-trajectory random permutation \\
fusion weight $\alpha_{\mathrm c}$ & 1.0 & \{1.0, 0.7, 0.5\} \\
step-weight floor $\lammix_0$ & 0.5 & \{0, 0.25, 0.5, 0.75, 1.0\}; fixed-weight sweep uses $\alpha_\lambda=0$ \\
modulation $\alpha_\lambda$ & 0.5 & \{0, 0.25, 0.5, 1.0\}; $\alpha_\lambda=0$ turns adaptation off \\
advantage normalization & per-benchmark & \{mean-norm, mean-and-std-norm\} \\
\bottomrule
\end{tabular}
\end{table}

\subsection{Evaluation protocol and published references}
\label{app:eval-protocol}
Table~\ref{tab:protocol} records the available evaluation settings. For the ALFWorld 1.5B run, each evaluated update supplies three success rates $s_{u,j}$, and the selected update is $u^\star=\arg\max_u\frac13\sum_{j=1}^3s_{u,j}$. The reported mean and population SD use the three scores at that same update. The final-update result uses $u=300$ independently of this maximization. This seed aggregation follows the testing convention stated by HGPO; the checkpoint selection described here is GACA's own protocol.

\begin{table}[H]
\caption{Evaluation and selection settings for the archived GACA ALFWorld 1.5B run.}
\label{tab:protocol}
\centering\small
\setlength{\tabcolsep}{5pt}
\renewcommand{\arraystretch}{1.12}
\begin{tabular}{@{}p{0.32\textwidth}p{0.60\textwidth}@{}}
\toprule
Setting & GACA, ALFWorld 1.5B \\
\midrule
Environment evaluation split & \texttt{eval\_in\_distribution} \\
Evaluation batch size & $128$ tasks \\
Response sampling & temperature $0.4$, top-$p=1$, top-$k$ disabled \\
Evaluation seeds & $42$, $123$, $2024$ \\
Evaluation frequency & every $5$ parameter updates \\
Checkpoint selection & maximum mean validation success over three seeds \\
Reported best / final update & $275$ / $300$ \\
Training runs underlying SD & one; SD is across evaluation seeds \\
History / episode limit & $2$ interactions / $50$ environment steps \\
\bottomrule
\end{tabular}
\end{table}

The original GiGPO study uses $150$ updates for ALFWorld and WebShop. HGPO reports $160$ training iterations; its updated-framework setting uses a $30$-step WebShop horizon, and its earlier-framework results use $15$ steps. Table~\ref{tab:main} retains representative configurations from both versions, and Table~\ref{tab:published-full} gives the full published-reference comparison, including both history-window lengths. Source results are carried over with their original mean and SD summaries. GACA's historical training log supplies the $150$-update and final-$300$ evaluations in addition to the selected best result.

\begin{table}[H]
\caption{Full success-rate comparison (\%, mean $\pm$ SD), including all published reference configurations used in this work. External rows are quoted from their sources; $K$ denotes the history-window length. HGPO's ALFWorld entries use its in-distribution setting. WebShop horizons and framework versions are indicated separately. GACA's ALFWorld deviations summarize three evaluation seeds at one selected checkpoint. These rows retain their source protocols.}
\label{tab:published-full}
\centering\small
\setlength{\tabcolsep}{6pt}
\renewcommand{\arraystretch}{1.12}
\begin{tabular}{@{}lcccc@{}}
\toprule
\multirow{2}{*}{\textbf{Method}} & \multicolumn{2}{c}{\textbf{ALFWorld}} & \multicolumn{2}{c}{\textbf{WebShop}} \\
\cmidrule(lr){2-3}\cmidrule(lr){4-5}
& \textbf{1.5B} & \textbf{7B} & \textbf{1.5B} & \textbf{7B} \\
\midrule
\multicolumn{5}{@{}l}{\textit{\citet{feng2026group}; WebShop: 15 steps}}\\
GRPO & $72.8$\,\sd{3.6} & $77.6$\,\sd{5.2} & $56.8$\,\sd{3.8} & $66.1$\,\sd{3.7} \\
GiGPO & $86.7$\,\sd{1.7} & $90.8$\,\sd{1.3} & $67.4$\,\sd{4.5} & $75.2$\,\sd{3.8} \\
\midrule
\multicolumn{5}{@{}l}{\textit{\citet{he2026hierarchy}, Table 7: earlier framework; WebShop: 15 steps}}\\
HGPO ($K=2$) & $89.58$\,\sd{0.45} & $91.15$\,\sd{1.19} & $72.66$\,\sd{1.78} & $76.43$\,\sd{1.47} \\
HGPO ($K=4$) & $92.45$\,\sd{0.81} & $94.79$\,\sd{0.90} & $75.91$\,\sd{1.19} & $77.21$\,\sd{0.22} \\
\midrule
\multicolumn{5}{@{}l}{\textit{\citet{he2026hierarchy}, Table 1: updated framework; WebShop: 30 steps}}\\
GiGPO ($K=2$) & $90.16$\,\sd{0.22} & $93.29$\,\sd{0.40} & $66.53$\,\sd{1.92} & $77.60$\,\sd{1.68} \\
HGPO ($K=2$) & $92.77$\,\sd{1.08} & $95.44$\,\sd{0.62} & $71.54$\,\sd{4.00} & $78.51$\,\sd{1.40} \\
GiGPO ($K=4$) & $93.29$\,\sd{1.07} & $95.63$\,\sd{0.49} & $73.24$\,\sd{2.25} & $77.79$\,\sd{0.92} \\
HGPO ($K=4$) & $94.85$\,\sd{0.92} & $95.96$\,\sd{0.49} & $78.12$\,\sd{2.06} & $79.29$\,\sd{1.17} \\
\midrule
\multicolumn{5}{@{}l}{\textit{This work; WebShop: 15 steps}}\\
\rowcolor{gacashade}
\textbf{GACA} & $\mathbf{96.88}$\,\sd{0.64} & $\mathbf{98.18}$\,\sd{0.37} & $\mathbf{80.99}$\,\sd{2.24} & $\mathbf{83.59}$\,\sd{3.38} \\
\bottomrule
\end{tabular}
\end{table}

\section{Implementation Details}
\label{app:impl}

\paragraph{E100 execution.} The configured score mode is named \texttt{entropy}, but its statistic is the sampled response's mean-token NLL in equation~\plaineqref{eq:nll}. The retention parameter \texttt{criticality\_top\_p=1.0} assigns every step to the high branch. Each step therefore forms its own segment, segment aggregation is the identity, and the stored $\beta_\lambda=0.5$ is unused. Equation~\plaineqref{eq:lambda} is the resulting gate. The equivalent centered parameterization is given below.

\paragraph{Normalization conventions.} Group statistics are computed over step rows. In the reconstructed ALFWorld 1.5B computation, each episode-side row carries $\widetilde R_t=R(\tau)-0.1\,\mathbf1\{\text{invalid action at }t\}$. For $N_g=\sum_{\tau\in g}T_\tau$, the statistics are $\mu_g=N_g^{-1}\sum_{\tau\in g}\sum_t\widetilde R_t$ and $\sigma_g^2=(N_g-1)^{-1}\sum_{\tau\in g}\sum_t(\widetilde R_t-\mu_g)^2$. Thus longer trajectories contribute more rows, and a local validity penalty can change the episode-side credit at the affected step. Discounted step returns are computed from environment rewards. Their anchor-group sample standard deviation uses divisor $|\mathcal G|-1$ for non-singleton groups. These conventions are retained in the implemented estimator inputs of Section~\ref{sec:theory}. The step advantage enters the mixture without additional multiplicative rescaling.

\paragraph{Loss reduction and regularization.} The archived ALFWorld 1.5B configuration uses \texttt{token-mean}: each valid response token has equal weight across the optimization batch. In equation~\plaineqref{eq:surrogate}, this corresponds to response weights $w_t=|a_t|/\sum_k|a_k|$ inside $\E_t$. The KL coefficient is $0.01$, using \texttt{low\_var\_kl}, and the full objective includes $-0.001\widehat{\mathcal H}$, where $\widehat{\mathcal H}$ is mean policy entropy over valid token contexts. This entropy regularizer is distinct from the sampled-response NLL score $H_t$.

\paragraph{Singleton step-group fallback.} A singleton anchor group admits no within-group contrast. We normalize that step's discounted return against the statistics of its \emph{prompt group}: $\Ast_t=\hat g_t-\mu_{\mathrm{prompt}}$ in mean-norm mode and $(\hat g_t-\mu_{\mathrm{prompt}})/(\sigma_{\mathrm{prompt}}+\varepsilon)$ in mean-and-std-norm mode, with the prompt-group statistics taken over all steps sharing the prompt index. This provides a group-relative reference for states observed only once, following the normalization conventions of equations~\plaineqref{eq:episode}--\plaineqref{eq:step}.

\paragraph{Per-step advantage scalar.} Every quantity entering \eqref{eq:final} is a per-step scalar: $\Ast_t$ and $\crit_t$ are computed per step, and the mixed $\Afin_t$ is broadcast back to the tokens of $a_t$ through the response mask. Where an advantage has already been broadcast, its per-step scalar is recovered by averaging over the step's valid response tokens.

\paragraph{Ablation controls and centered gate.} In the fixed-weight sweep, $\alpha_\lambda=0$ gives $\lammix_t=\lammix_0$ and leaves the advantage definitions and singleton fallback unchanged. The modulation-off result at $\lammix_0=0.5$ is the same configuration as the $0.5$ fixed sweep. With $\lambda_{\mathrm{ref}}=\lammix_0(1+\alpha_\lambda)$ and $\kappa_\lambda=\lammix_0\alpha_\lambda$, the gate can be written $\lambda_t=\min\{\lambda_{\mathrm{ref}}+\kappa_\lambda(c_t-1),1\}$. For $c_t\equiv1$, the default gate has constant weight $0.75$. The observed step-weighted training mean is denoted separately by $\bar\lambda_{\rm train}$. The random control permutes NLL scores within each trajectory, preserving its weight multiset while changing their pairing with the two advantage signals. The original GiGPO rule $A^{\mathrm E}+A^{\mathrm S}$ also has a different total advantage scale from the fixed $\lambda=0.5$ GACA mixture.

\section{Computational Cost}
\label{app:cost}
GACA adds no parameters, no rollouts and no forward passes: the per-token log-probabilities that equation~\plaineqref{eq:nll} averages are already materialized for the importance ratio in equation~\plaineqref{eq:surrogate}, and everything downstream of them is elementwise arithmetic on per-step scalars. Table~\ref{tab:cost} confirms this on the setting where the horizon is longest. The two components GACA adds over GiGPO cost $0.09$\,s per iteration together, against an iteration dominated by rollout and policy update, giving a reported overhead below $0.2\%$ over GRPO in this component-level timing comparison.

\begin{table}[H]
\caption{Per-iteration wall-clock breakdown on ALFWorld with Qwen2.5-1.5B-Instruct, averaged over 20 training iterations on the same hardware. Rows above the rule are shared by all three methods; rows below are the additions each method makes to the advantage computation. The reported peak memory is the same for all three configurations.}
\label{tab:cost}
\centering
\small
\setlength{\tabcolsep}{9pt}
\renewcommand{\arraystretch}{1.12}
\begin{tabular}{@{}lccc@{}}
\toprule
\textbf{Component (s / iteration)} & \textbf{GRPO} & \textbf{GiGPO} & \textbf{GACA} \\
\midrule
Environment rollout & \multicolumn{3}{c}{$318.6$} \\
Old- and reference-policy log-probabilities & \multicolumn{3}{c}{$23.1$} \\
Clipped policy update & \multicolumn{3}{c}{$29.5$} \\
\midrule
Anchor-state grouping (hash lookups) & --- & $0.01$ & $0.01$ \\
Step-level advantage & --- & $0.53$ & $0.53$ \\
sampled NLL and surprisal score & --- & --- & $0.07$ \\
Mixing weight and advantage blend & --- & --- & $0.02$ \\
\midrule
\rowcolor{gacashade}
Total & $371.2$ & $371.7$ & $371.8$ \\
\rowcolor{gacashade}
Overhead vs.\ GRPO & --- & $+0.15\%$ & $+0.17\%$ \\
\rowcolor{gacashade}
Peak GPU memory (GB) & $62.4$ & $62.4$ & $62.4$ \\
\bottomrule
\end{tabular}
\end{table}

\section{Complete Ablation Results}
\label{app:suppablation}
Table~\ref{tab:ablation} gives the complete ALFWorld 1.5B ablation grid underlying the main allocation comparison and the additional sensitivity analyses below. The two fixed-$0.5$ entries are one configuration and one result.

\begin{table}[H]
\caption{Ablation grid on ALFWorld (1.5B): each block varies one design choice with the remaining components at their defaults, under a single training and selection protocol and averaged over the same three evaluation seeds. Best per block in \textbf{bold}; the full default configuration is shaded. The two entries marked $\dagger$ are the same configuration: $\alpha_\lambda=0$ and $\lammix_0=0.5$ give a fixed equally weighted mixture.}
\label{tab:ablation}
\centering
\small
\setlength{\tabcolsep}{12pt}
\renewcommand{\arraystretch}{1.12}
\begin{tabular}{@{}lcc@{}}
\toprule
\textbf{Variant} & \textbf{Success rate} & \textbf{$\Delta$ vs.\ default} \\
\midrule
\multicolumn{3}{@{}l}{\textbf{Criticality score.}\ \ \textit{Does score-to-step assignment matter?}}\\
\rowcolor{gacashade}
\quad NLL proxy \textit{(default)}   & $\mathbf{96.88}$ & \textit{reference} \\
\quad random permutation           & $88.54$          & \textcolor{gacagray}{$-8.34$} \\
\addlinespace[3pt]
\multicolumn{3}{@{}l}{\textbf{Fixed mixing weight $\lammix_0$, adaptation off.}\ \ \textit{Which tested constant is best?}}\\
\quad $0.0$ (episode-level only)   & $85.94$          & \textcolor{gacagray}{$-10.94$} \\
\quad $0.25$                       & $91.20$          & \textcolor{gacagray}{$-5.68$} \\
\quad $0.5^\dagger$                & $\mathbf{93.05}$ & \textcolor{gacagray}{$-3.83$} \\
\quad $0.75$                       & $92.41$          & \textcolor{gacagray}{$-4.47$} \\
\quad $1.0$ (step-level only)      & $84.60$          & \textcolor{gacagray}{$-12.28$} \\
\addlinespace[3pt]
\multicolumn{3}{@{}l}{\textbf{Modulation strength $\alpha_\lambda$.}\ \ \textit{Which strength is best?}}\\
\quad $0$ (no modulation)$^\dagger$ & $93.05$         & \textcolor{gacagray}{$-3.83$} \\
\quad $0.25$                        & $95.12$         & \textcolor{gacagray}{$-1.76$} \\
\rowcolor{gacashade}
\quad $0.5$ \textit{(default)}      & $\mathbf{96.88}$ & \textit{reference} \\
\quad $1.0$                         & $95.44$         & \textcolor{gacagray}{$-1.44$} \\
\addlinespace[3pt]
\multicolumn{3}{@{}l}{\textbf{Fusion weight $\alpha_{\mathrm c}$.}\ \ \textit{Does return change help?}}\\
\rowcolor{gacashade}
\quad $1.0$ \textit{(default)}      & $\mathbf{96.88}$ & \textit{reference} \\
\quad $0.7$                         & $96.05$         & \textcolor{gacagray}{$-0.83$} \\
\quad $0.5$                         & $95.20$         & \textcolor{gacagray}{$-1.68$} \\
\bottomrule
\end{tabular}
\end{table}

\subsection{Modulation strength and the mixing law}
At $\lambda_0=0.5$, the tested strengths $\alpha_\lambda=0,0.25,0.5,1$ give success rates of $93.05\%$, $95.12\%$, $96.88\%$, and $95.44\%$. All three nonzero strengths have higher mean success than the fixed mixture, with the default $0.5$ giving the largest mean. Increasing $\alpha_\lambda$ raises the weight at a given positive score and lowers the criticality threshold for saturation. The sweep therefore varies the mean weight together with sensitivity to criticality and the extent of clipping.

\subsection{Choice of criticality signal}
We also test adding normalized changes in discounted return to the NLL score, keeping the gate parameters fixed.

\paragraph{NLL--return-change fusion.} In the measured non-default control, let $\hat e_t=\Hstep_t/\overline{\Hstep}$ and $\hat d_t=\Delta_t/\overline{\Delta}$ with $\Delta_t=|\hat g_t-\hat g_{t-1}|$ and $\Delta_1=0$. When $\overline{\Delta}>0$, the optional fusion is
\begin{equation}
\mathrm{score}_t=\alpha_{\mathrm c}\hat e_t+(1-\alpha_{\mathrm c})\hat d_t,\qquad \alpha_{\mathrm c}\in[0,1].
\label{eq:fuse}
\end{equation}
The fused score is normalized within the completed trajectory before gating; $\alpha_{\mathrm c}=1$ gives the main NLL-only method. The archived configuration does not specify the handling of $\overline\Delta=0$.

Pure NLL ($\alpha_{\rm c}=1$) reaches $96.88\%$, compared with $96.05\%$ and $95.20\%$ for NLL shares $0.7$ and $0.5$, respectively (Table~\ref{tab:ablation}), supporting the NLL-only choice among these tested constructions.

\section{Supplementary Diagnostics}
\label{app:diagnostic-protocol}
The diagnostics use the frozen ALFWorld 1.5B final@300 checkpoint with temperature $1$, top-$p=1$, top-$k$ disabled, and at most $512$ response tokens. The environment uses success reward $10$, discount $0.95$, a $50$-step episode limit, and a two-step history window. The weight-distribution sample uses $16$ fixed training tasks. The paired diagnostic uses $32$ tasks stratified by training-task type, fixed before sampling with master seed $2026092602$.

\subsection{Training mean and frozen-policy weights}
\label{app:diagnostic-weights}
Across $300$ parameter updates, the logged mixing weights have a step-weighted mean of approximately $0.748$; each update's mean is weighted by its number of environment steps. Trajectory-normalized scores have mean one, giving the default gate a pre-cap mean of $0.75$, with clipping slightly reducing it. This motivates the fixed-$0.75$ training control in Section~\ref{sec:ablations}.

Stepwise weight distributions use newly collected complete rollout groups from the frozen final@300 checkpoint. Each of $16$ fixed training tasks supplies $8$ trajectories; $122$ of the $128$ trajectories succeed. The sample contains $1{,}269$ environment steps, with non-singleton anchor groups covering $96.30\%$ of the steps.

The sampled weight mean is $0.7484$, its median is $0.7390$, and its 10th--90th percentiles are $0.6228$--$0.8996$; $2.68\%$ of steps reach $\lambda_t=1$ (Figure~\ref{fig:weight-distribution}).

\begin{figure}[H]
\centering
\includegraphics[width=0.60\textwidth]{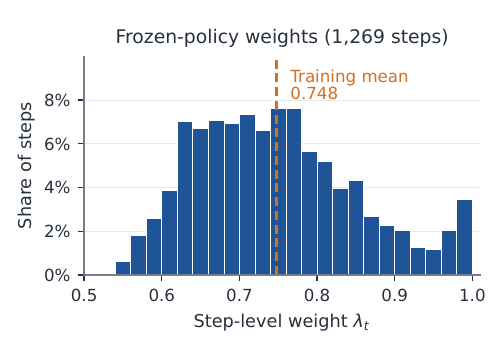}
\caption{Step weights from the frozen ALFWorld 1.5B final@300 policy, using $128$ trajectories and $1{,}269$ decisions on training tasks. The vertical line marks the historical training mean.}
\label{fig:weight-distribution}
\end{figure}

\subsection{Within-trajectory score and action diversity}
\label{app:diagnostic-diversity}
A full pre-action state includes the task, model input, interaction history, and current environment state. Fresh responses sampled independently of the original rollout estimate action diversity. If $n_u$ of the $n$ samples execute canonical action $u$, the disagreement estimator is
\begin{equation}
\widehat D(s)=1-\frac{\sum_u n_u(n_u-1)}{n(n-1)}.
\label{eq:diagnostic-diversity}
\end{equation}
The parser groups responses by environment-executed action identity. Invalid responses share an identity only when their environment behavior is identical.

\paragraph{Paired sampling.}
For each of the $32$ stratified tasks, we collect eight trajectories, totaling $256$ trajectories and $2{,}181$ decisions. One trajectory per task is selected uniformly, followed by two decision steps drawn uniformly without replacement. A step in a trajectory of length $T$ has selection probability $(1/8)(2/T)$. The two selected steps are ordered by their original rollout scores $c_t$; all $32$ pairs have distinct scores and mixing weights. At each state, $32$ fresh responses estimate diversity, totaling $2{,}048$ responses. Forty-one of the $64$ states have zero observed disagreement. Figure~\ref{fig:mechanism} displays all $32$ pairs: $11$ diversity differences are positive, $5$ negative, and $16$ zero.

Task pairs are the units of analysis and resampling. The $95\%$ percentile intervals in Table~\ref{tab:paired-diagnostics} use $2{,}000$ task-bootstrap replicates with seed $2026092602$, preserving both states in each sampled task. They describe uncertainty across this task panel conditional on the frozen checkpoint. The action-diversity difference remains positive after each leave-one-task-out exclusion ($0.0569$--$0.0834$). Since trajectory normalization divides both scores by the same positive value, original NLL and criticality induce the same within-trajectory ordering.

\subsection{Action-value measurements}
\label{app:diagnostic-value}
The paired diagnostic uses the same $64$ states for continuation-value measurements. At each state, $m=8$ independent response draws give parsed actions $U_i$, each followed by independent continuation blocks $b=1,2$. All $1{,}024$ candidate branches pass restoration checks on observations, admissible actions, model-input tokens, and environment-state hashes. Discounted returns include the initial action and use the same remaining horizon. ALFWorld stores the parsed action in subsequent policy history, so the continuation context depends on the canonical action.

Within each block, returns from draws executing the same action are averaged. Let $\widetilde G_i^{(b)}$ be the resulting mean assigned to draw $i$, retaining the original action multiplicities, and let $\overline{\widetilde G}^{(b)}=m^{-1}\sum_i\widetilde G_i^{(b)}$. We estimate action-value variation by
\begin{equation}
\widehat C(s)=\frac{1}{m-1}\sum_{i=1}^{m}
\bigl(\widetilde G_i^{(1)}-\overline{\widetilde G}^{(1)}\bigr)
\bigl(\widetilde G_i^{(2)}-\overline{\widetilde G}^{(2)}\bigr).
\label{eq:diagnostic-heterogeneity}
\end{equation}
Independent continuation blocks remove noise cross-products in expectation. Shared-action averaging gives identical estimates to repeated draws of the same action; states with a single sampled action therefore have zero estimated variation. The paired sample has $47$ such states, and two of the remaining state estimates are negative. Negative finite-sample estimates are retained.

\begin{table}[H]
\caption{Paired measurements ordered by original rollout criticality. Means use all $32$ task pairs. Differences are higher-score minus lower-score; intervals use task resampling. The value estimates use same-action averaging in each continuation block.}
\label{tab:paired-diagnostics}
\centering\small
\setlength{\tabcolsep}{6pt}
\begin{tabular}{@{}lrrrr@{}}
\toprule
Quantity & Lower $c$ & Higher $c$ & Difference & $95\%$ interval \\
\midrule
Action diversity $\widehat D$ & 0.1041 & 0.1789 & 0.0747 & [0.0248, 0.1334] \\
Action-value variation $\widehat C$ & 0.2502 & 0.1710 & $-0.0792$ & [$-0.3599$, 0.1954] \\
\bottomrule
\end{tabular}
\end{table}

The interval for the action-value difference includes zero, so this panel provides no clear evidence of value separation between the higher- and lower-score states.

\end{document}